\RequirePackage{amsmath,amssymb}
\documentclass[runningheads,a4paper]{llncs}
\pdfoutput=1

\usepackage{times}

\usepackage{graphicx}
\usepackage{algorithm,algorithmicx,algpseudocode,float}
\usepackage{color}
\usepackage{subfigure}

\newcommand{\tirage}[1]{\widetilde{#1}}

\newcommand{\probacond}[3]{P_{#3}(#2\ |\ #1)}
\newcommand{\probacondsign}[1]{P_{#1}}
\newcommand{\ACincons}[1]{\breve{#1}}

\newcommand{\varAlgo}[1]{\mathtt{X}[#1]}
\newcommand{\domAlgo}[1]{\mathtt{D}[#1]}
\newcommand{\cstrAlgo}[1]{\mathtt{C}[#1]}
\newcommand{\remAlgo}[1]{\mathtt{R}[#1]}
\newcommand{\solAlgo}[3]{\mathtt{S_{#1}}[#2][#3]}

\newcommand{\espAlgo}[1]{\mathtt{E}[#1]}
\newcommand{\espNetwork}[1]{\mathtt{EN}}
\newcommand{\probaAlgo}[2]{\mathtt{P}[#1][#2]}

\newcommand{\tmp}{\mathtt{tmp}}

\newtheorem{coroll}{Corollary}

\title{A Probabilistic-Based Model for Binary CSP}
\author{
Amine Balafrej\inst{1} \and Xavier Lorca\inst{1} \and Charlotte Truchet\inst{2}}

\institute{TASC - Mines-Nantes, INRIA, LINA UMR 6241\\
\email{\{Amine.Balafrej,Xavier.Lorca\}@mines-nantes.fr}
\and CELTIQUE, IRISA - UMR 6074, Rennes, France\\
 \email{Charlotte.Truchet@univ-nantes.fr}
}

\begin{document}
\mainmatter  

\maketitle

\begin{abstract}
This work introduces a probabilistic\nobreakdash-based model for binary CSP that provides a fine grained analysis of its internal structure. 
Assuming that a domain modification could occur in the CSP, it shows how to express, in a predictive way, the probability that a domain value becomes inconsistent, 
then it express the expectation of the number of arc-inconsistent values in each domain of the constraint network. 
Thus, it express the expectation of the number of arc-inconsistent values for the whole constraint network. 
Next, it provides bounds for each of these three probabilistic indicators.
Finally, a polytime algorithm, which propagates the probabilistic information, is presented.
\end{abstract}

\section{Introduction} \label{intro}
The core of constraint programming, i.e. its operational nature, depends on the propagation\nobreakdash-research mechanism: the propagation part tries to infer new information from the current states of variables, while the search part, most of the time, consists of a depth-first exploration of the search space.
Propagation and search must generally be intricated because of the NP-completeness of the CSPs.
Thus, finding a fair balance between efficiency, in terms of calculation time, and effective performance in terms of filtering, has always been a major issue in the constraint programming community.

This paper goes one step further by reporting a probabilistic analysis of the constraint network associated with each constraint satisfaction problem (CSP). This leads to a probabilistic\nobreakdash-based model for binary CSP that allows us to better understand both the macro-structure (i.e., interactions between variables through the constraints) and the micro-structure as defined in \cite{Jegou93a} (i.e., interactions between compatible values) of a binary CSP.
The contribution of this paper consists on a theoretical analysis of the constraint networks from a probabilistic point of view:
\begin{enumerate}
\item it is shown how to compute in a predictive way the probability for each value of a domain to be arc\nobreakdash-inconsistent, under the hypothesis that a domain modification occurs in the CSP;  
\item next, it is demonstrated how to aggregate this information for the whole domain and for the whole constraint network; 
\item then, these results are approximated by lower bounds; 
\item finally, a polytime algorithm is proposed to compute these probabilistic informations.
\end{enumerate}


\section{Background material and notations} \label{background}
We consider the classical definition of binary constraint networks.
A binary constraint network ${N}$ is a triplet $\nobreak{<\mathcal{X},\mathcal{D},\mathcal{C}>}$, where $\mathcal{X}$ is a set of $n$ variables, $\mathcal{D}$ the set of their finite domains, and $\mathcal{C}$ the binary constraints, which are assumed to be unique without loss of generality. 
We write:
$C_{ij}$ the constraint between $X_i$ and $X_j$, $\mathcal{C}(X_i)$ the set of constraints involving $X_i$, and $\Gamma(X_i)=\{X_j | \exists C_{ij} \in \mathcal{C}\}$;
$S_{ij}$ the set of solutions of $C_{ij}$ alone;
Given a value $w\in D_i$, $S_{ij}^{w}=\{(w,v_j) | (w, v_j) \in S_{ij} \}$ the supports of $w$;
$S$ the set of solutions of the network, that is, the values in $D_1 \times... \times D_n$ satisfying all the constraints;
$\pi_i$  the $i$-th projection of $\mathbb{N}^n$.
%
\emph{Constraint propagation} aims at detecting values in the domains that cannot satisfy at least one constraint. Propagation is based on the consistency property, for which there are several, more or less powerful, definitions.
\begin{definition}[Arc-consistency or AC]\label{arcConsistency}
A \emph{value} $v_i \in D_i$ is AC for $C_{ij}$ if and only if
$\exists v_j \in D_j$, s.t. $v_j\in\pi_j(S_{ij}^{v_i})$.
A \emph{domain} $D_i$ is AC on $C_{ij}$ if and only if 
$D_i \neq \emptyset$ and $\forall v_i \in D_i$, $v_i$ is AC on $C_{ij}$.
A domain $D_i$ is AC if and only if it is AC on any $C_{ij}\in \mathcal{C}(X_i)$. 
A \emph{constraint network} $N=<\mathcal{X}, \mathcal{D}, \mathcal{C}>$ is AC if and only if any $D_i \in \mathcal{D}$ is AC.
\end{definition}
Consequently, a value $w \in D_i$ is \emph{arc-inconsistent} on a constraint $C_{ij}$ if and only if $\nobreak{\pi_j(S_{ij}^{w})=\emptyset}$. 
Such a value is written $\ACincons{w}$. 

\section{A probabilistic\nobreakdash-based model for CSP}  
\label{probaModel}
Constraint networks are difficult to analyze as a whole. 
Solving methods often focus on the state of one variable inside the network, which is called the \emph{microstructure}. 
In addition, solving methods use criteria based on statistical information based on the past states of variable/constraint.
However, such a point of view looses much information, since no information on the future state is considered and the global structure of the network is ignored.
This section introduce an original probabilistic model for binary CSP that allows us to define criteria based on the future state of variables considering both the \emph{macrostructure} and the \emph{microstructure} of the constraint network.

Assuming that a domain will be modified, we want to know which values will become more likely to be arc-inconsistent.
We first introduce a probabilistic\nobreakdash-based model of the network, which allows us to properly define the domain modifications 
as probabilistic events. Then, we give a calculation of the probability for a value to be arc-inconsistent after a domain modification (here, removing a fixed number of values), considering the whole constraint network. 
Let $N=<\mathcal{X}, \mathcal{D}, \mathcal{C}>$ a constraint network, we build a 
\emph{probabilized network} $\tirage{N}$, from $N$ by associating to each domain $D_i$ a \emph{random variable} $\tirage{D}_i$ such that $\tirage{D}_i \subset D_i$.
All these random variables are drawn independently, i.e., they are randomly and uniformly chosen as a fixed length subdomain of $D_i$. In this network, we are interested in particular events: the domain modifications.  We will consider what happens when a domain $D_i$ is reduced. Knowing that $k$ values have been removed  in $D_i$, we randomize \emph{which} $k$ values have been removed from $D_i$. The number of such values is denoted $r_k(D_i)$.

\subsection{Probabilistic model for arc-inconsistency} \label{ProbaNetwork}
First, we detail how to compute the probability for a value $w$ of a given variable $X_i$ to be arc-inconsistent on a constraint $C_{ij}$ given an event $r_k(D_j)$. 
Second, we go one step further by evaluating the expectation of the number of arc-inconsistent values,
in a domain $D_i$ of a variable $X_i$, according to a potential event, $r_{k_j}(D_j)$, occurring on the domain of a variable $X_j$ in the neighborhood of $X_i$.

\begin{figure*}[t]
\begin{center}
    \subfigure[{\small Initial state}]{%
      \includegraphics[width=0.28\textwidth]{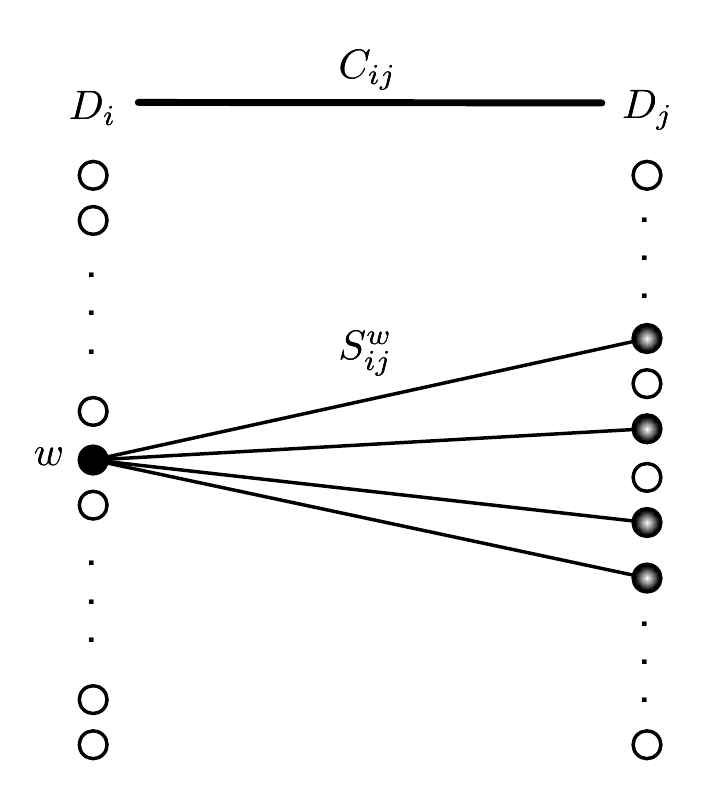} \label{main}
    }
    \subfigure[{\small $k<4$}]{%
      \includegraphics[width=0.33\textwidth]{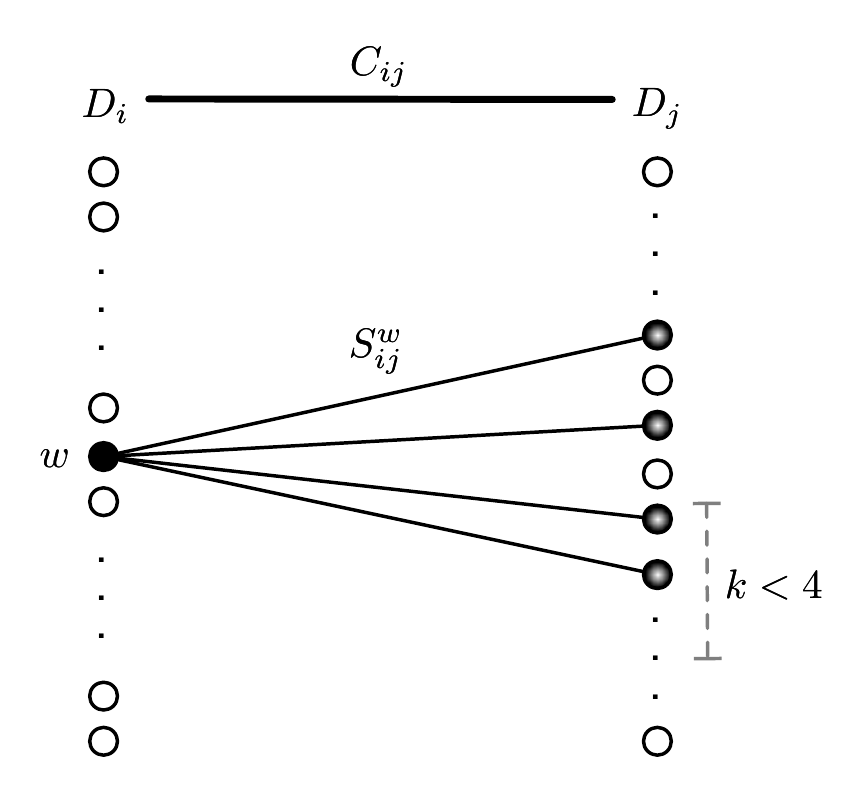} \label{nullProb}
    }
    \subfigure[{\small $k\geq4$}]{%
      \includegraphics[width=0.33\textwidth]{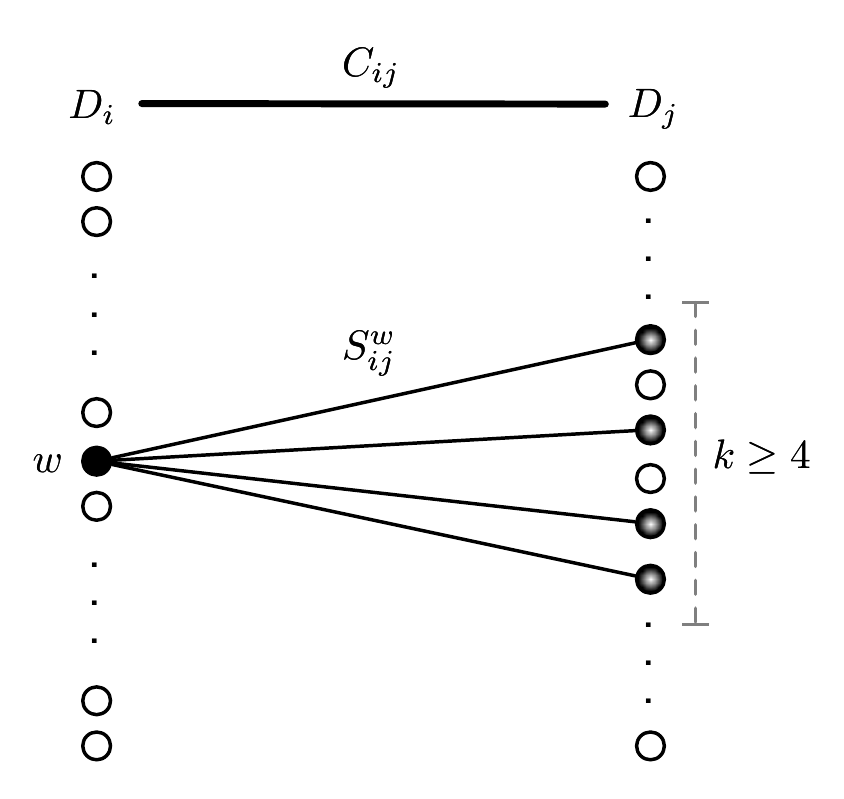} \label{nonNullProb}
    }
\end{center}
    \caption{Probability of being arc-inconsistent on a constraint $C_{ij}$ for a single value $w \in D_i$ according to a possible distribution of $k$ values removed from $D_j$.}
    \label{figure:probOneValue}
 \end{figure*}

Consider the example provided by Figure~\ref{figure:probOneValue}, it first depicts the supports $S^{w}_{ij}$ of a value $w\in D_i$ on a constraint $C_{ij}$ (Figure~\ref{main}). An interesting question to predict the \emph{importance} of the value $w$ could be its capacity to be arc-inconsistent. Then, a basic information has to be formalized: 
if it is assumed that at most $3$ values could be removed from $D_j$ (Figure~\ref{nullProb}) then, there is no chances for value $w$ to be arc-inconsistent on $C_{ij}$ (none of the possible combinations of $k<4$ values in $D_j$ could remove all the supports of $w$);
Otherwise, if it is assumed that at least $4$ values could be removed from $D_j$ (Figure~\ref{nonNullProb}) then, there is a chance for a value $w$ to be arc-inconsistent on $C_{ij}$ and we want to evaluate this chance.
From a probabilistic point of view, this information can be translated into the probability for $w$ of being arc-inconsistent on $C_{ij}$ according to value(s) removal(s) in $D_j$. 
In the following the projection $\pi_j(S_{ij}^{w})$ will be denoted by $\pi_j^{w}$.

\begin{proposition} {\label{probaValue}}
For a value $w \in D_{i}$, the probability of being arc-inconsistent on a constraint $C_{ij}$ in the probabilized constraint network $\tirage{N}$, knowing that $k$ values have been removed from $D_j$, is: 
\begin{equation}\nonumber
\probacond{r_k(D_j)}{\ACincons{w}}{c_{ij}}=\left\{
\begin{array}{l l r}
1,& \text{if }|\pi_j^{w}|=0 & (a) \\
& & \\
\prod\limits_{\ell=1}^{|\pi_j^{w}|}\frac{k-|\pi_j^{w}|+\ell}{|D_j|-|\pi_j^{w}|+\ell}, & \text{if }k\geq|\pi_j^{w}|>0 & (b)\\
& & \\
0,& \text{if }|\pi_j^{w}|>k\geq 0 & (c)\\
\end{array}
\right.
\end{equation}
\end{proposition}

\begin{proof}
We first recall that by definition we have $|D_j|\geq |\pi_j^{w}|$ and $|D_j|\geq k\geq 0$.
For the cases $(a)$ and (c) the proof is direct from Definition~\ref{arcConsistency} (precisely arc-inconsistency for (a)).
In order to build the proof for the case $(b)$, we introduce, for a value $w \in D_i$, the concept of \emph{k\nobreakdash-Support} which denotes any subset of $k$ values of a given domain $D_j$ containing all the support values for $w$ on the constraint $C_{ij}$. To choose a  k-Support, you only need to choose $(k-|\pi_j^{w}|)$ values outside the support values for $w$, hence the number of k-supports is:

\begin{equation}\nonumber
\emph{\#k-Supports}=\binom{|D_j|-|\pi_j^{w}|}{k-|\pi_j^{w}|}
\end{equation}

Then, for a value $w\in D_{i}$, the probability of being arc-inconsistent on a constraint $C_{ij}$ knowing that $k$ values have been removed from the domain $D_j$ is the probability of removing one of the \emph{k-supports} of $w$ on $C_{ij}$, thus 

\begin{equation}\nonumber
\probacond{r_k(D_j)}{\ACincons{w} }{c_{ij}} = \frac{\binom{|D_j|-|\pi_j^{w}|}{k-|\pi_j^{w}|}}{\binom{|D_j|}{k}}\\ 
\end{equation}

By developing the binomials, the fraction $\frac{\binom{|D_j|-|\pi_j^{w}|}{k-|\pi_j^{w}|}}{\binom{|D_j|}{k}}$ leads to $\frac{k!}{(k-|\pi_j^{w}|)!}\times\frac{(|D_j|-|\pi_j^{w}|)!}{|D_j|!}$
and consequently to 
$\prod\limits_{\ell=1}^{|\pi_j^{w}|}\frac{k-|\pi_j^{w}|+\ell}{|D_j|-|\pi_j^{w}|+\ell}$
\qed
\end{proof}

\begin{figure}
  \center
\includegraphics[width=0.7\linewidth]{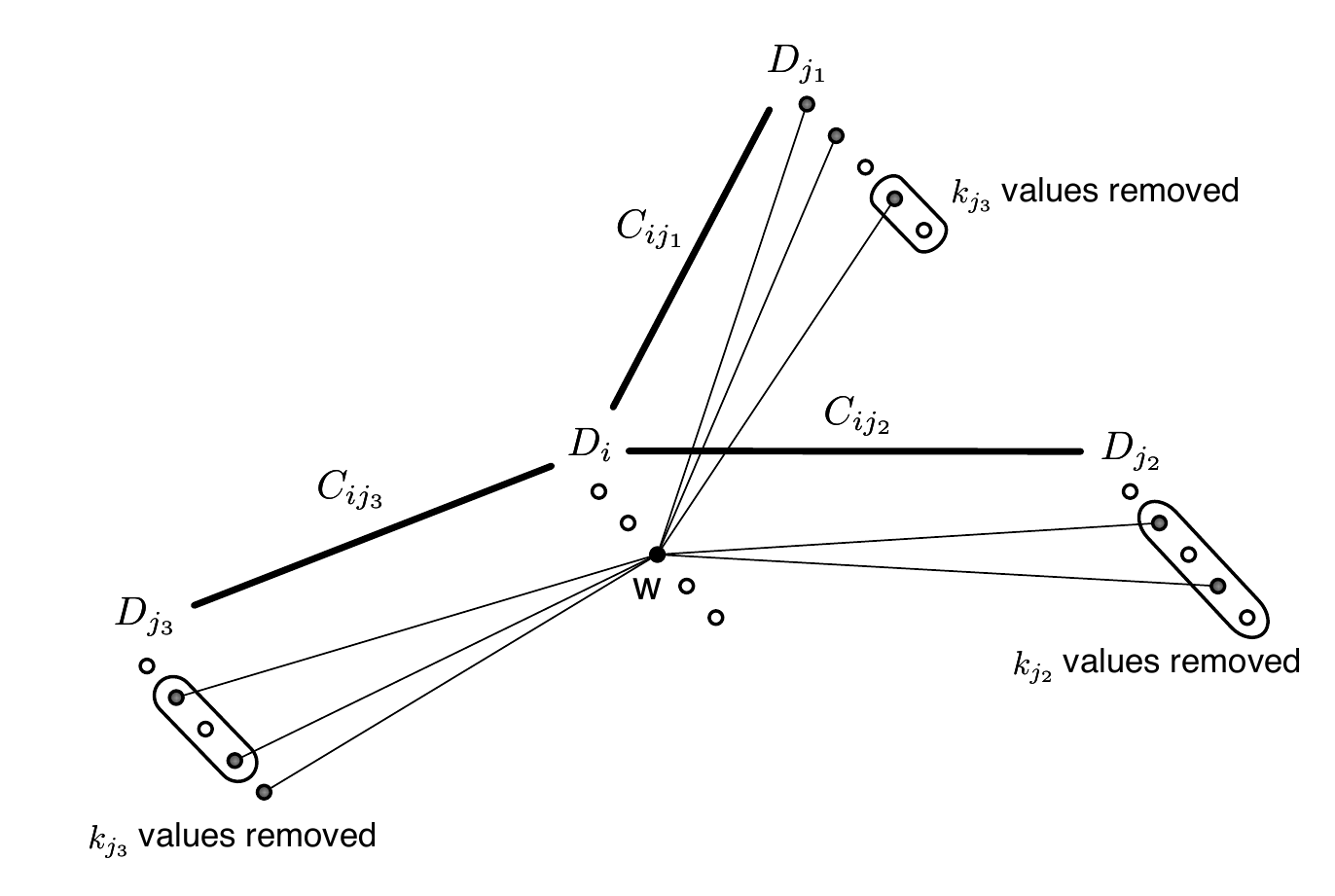}
  \caption{The probability of being arc-inconsistent for a value in the constraint network}\label{figure:probOneValueNetwork}
\end{figure}

An interesting information for the whole constraint network is the probability, for a value $w\in D_i$, to be arc-inconsistent for any constraint involving $D_i$.
Figure~\ref{figure:probOneValueNetwork} depicts an example of variable domain $D_i$ involved in three constraints 
with an event $r_{k_j}(D_j)$ in the domain of each variable $X_j\in\Gamma(X_i)$. 
For each event $r_{k_j}(D_j)$, a value $w\in D_i$ has a probability $\probacond{r_{k_j}(D_j)}{\ACincons{w} }{c_{ij}}$ of being arc-inconsistent on the constraint $C_{ij}$.
All these probabilities are aggregated and Proposition~\ref{probaValueAllConstr} provides the probability of being arc-inconsistent for a value $w\in D_{i}$ beyond the constraints themselves.

\begin{proposition} \label{probaValueAllConstr}
For a value $w \in D_{i}$, the probability of being arc-inconsistent in the probabilized constraint network $\tirage{N}$, knowing that $k_j$ values have been removed from each $D_j \in \Gamma(X_i)$, denoted $\probacond{\forall X_j\in\Gamma(X_i), r_{k_j}(D_j)}{\ACincons{w} }{N}$,  is equal to 

\begin{equation}
1-(\prod\limits_{X_j\in\Gamma(X_i)}^{}(1-\probacond{r_{k_j}(D_j)}{\ACincons{w} }{c_{ij}}))
\end{equation}
\end{proposition}

\begin{proof}
 For a value $w\in D_{i}$, the probability of being arc-consistent on a constraint $C_{ij}$, knowing the event $r_k(D_j)$, is $1-\probacond{r_{k_j}(D_j)}{\ACincons{w} }{c_{ij}}$. 
Then, the probability of being arc consistent knowing that $k_j$ values have been removed from each $D_j \in \Gamma(X_i)$,  is equal to:

\begin{equation}
 \prod\limits_{X_j\in\Gamma(X_i)}^{}(1-\probacond{r_{k_j}(D_j)}{\ACincons{w} }{c_{ij}}) \nonumber
\end{equation}
And so, the probability of being arc-inconsistent is equal to

\begin{equation}
1-(\prod\limits_{X_j\in\Gamma(X_i)}^{}(1-\probacond{r_{k_j}(D_j)}{\ACincons{w} }{c_{ij}})) \nonumber
\end{equation}
\qed
\end{proof}

Once we are able to express the probability of being arc-inconsistent for a value $w\in D_i$, 
we want to evaluate the number of arc-inconsistent values we expect found in $D_i$.
Propositions~\ref{expectationDomain} express this expectation for a single domain and Propositions~\ref{expectationNetwork} generalises this result for the whole constraint network. 

\begin{proposition}\label{expectationDomain}
 The expected number of arc-inconsistent values in a domain $D_i$ knowing that $k_j$ values have been removed from each $D_j\in\Gamma(X_i)$, denoted $E(\tirage{D}_{i})$, is
\begin{equation}
E(\tirage{D}_i)=\sum\limits_{w_i\in D_i} \probacond{\forall X_j\in\Gamma(X_i), r_{k_j}(D_j)}{\ACincons{w}_i \in D_i}{N}
\end{equation}
\end{proposition}

\begin{proof}
For each value $w\in D_i$, we define the random variable $Y_w$ s.t.
\begin{equation}
\nonumber
Y_w=\left\{
\begin{array}{l l}
1 & \text{ if $w$ is inconsistent after reduction of $(D_j)$ for a $X_j\in\Gamma(X_i)$ }\\ 
  & \\
0 & \text{ otherwise}\\
\end{array} \right.
\end{equation}
Note that $Y_w$ can take value $1$ with probability $\probacond{\forall X_j\in\Gamma(X_i), r_{k_j}(D_j)}{\ACincons{w} }{N}$ and value $0$ with probability $1-\probacond{\forall X_j\in\Gamma(X_i), r_{k_j}(D_j)}{\ACincons{w} }{N}$. Thus,
\begin{equation}
\nonumber
\begin{array}{r l}
E(Y_w)	=&1\times \probacond{.}{\ACincons{w} }{N} + 0\times (1-\probacond{.}{\ACincons{w} }{N}) \\ 
=& \probacond{\forall X_j\in\Gamma(X_i), r_{k_j}(D_j)}{\ACincons{w} }{N}
\end{array}
\end{equation}
The number of values in $D_i$ that are arc-inconsistent knowing that $k_j$ values have been removed from each $D_j$ is $\sum\limits_{w\in D_i}Y_w$:
\begin{align}
E(\tirage{D}_{i}) = E(\sum\limits_{w\in D_i}Y_w) = \sum\limits_{w\in D_i}E(Y_w) \nonumber
\end{align}
Consequently:
\begin{equation} \label{eq1}
E(\tirage{D}_i)=\sum\limits_{w_i\in D_i} \probacond{\forall X_j\in\Gamma(X_i), r_{k_j}(D_j)}{\ACincons{w}_i \in D_i}{N}
\end{equation}
\qed
\end{proof}

\begin{proposition}\label{expectationNetwork}
 The expectation of the number of arc-inconsistent values for the whole network, denoted $E(\tirage{N})$, is:
\begin{equation} \label{eq2}
E(\tirage{N})= \sum\limits_{D_i\in\mathcal{D}} \sum\limits_{w_i\in D_i} \probacond{\forall X_j\in\Gamma(X_i), r_{k_j}(D_j)}{\ACincons{w}_i \in D_i}{N} 
\end{equation}
\end{proposition}
\begin{proof}
Similar to those of Proposition~\ref{expectationDomain}.
\qed
\end{proof}

Now we have a complete probabilistic\nobreakdash-based model that allows us, given an event, to compute the probability of a value of being arc-inconsistent for a single constraint (Prop.~\ref{probaValue}) in addition to three interesting probabilistic indicator  
that are:  
\begin{itemize}
 \item The probability for each value of being arc-inconsistent (Prop.~\ref{probaValueAllConstr}).
 \item The expected number of arc-inconsistent values for a given domain (Prop.~\ref{expectationDomain}).
 \item The expected number of arc-inconsistent values for the whole constraint network (Prop.~\ref{expectationNetwork}).
\end{itemize}
The following corollaries provide bounds for each of these three probabilistic indicators. 


\begin{coroll} \label{probaValueNet}
For $w\in D_{i}$, the probability of being arc-inconsistent in $\tirage{N}$, knowing that $k_j$ values have been removed from each domain $D_j$ of each variable $X_j\in\Gamma(X_i)$, is greater or equal than the maximum among the probabilities on each constraint:
\begin{equation}
\probacond{\forall X_j\in\Gamma(X_i), r_{k_j}(D_j)}{\ACincons{w} }{N} \geq \max_{\ X_j\in\Gamma(X_i)}^{} \{\probacond{r_{k_j}(D_j)}{\ACincons{w} }{c_{ij}} \} 
\end{equation}
\end{coroll}

\begin{proof}
Knowing that,
\begin{equation} \label{supA}
\forall X_j\in\Gamma(X_i),\ 0 \leq \probacond{r_{k_j}(D_j)}{\ACincons{w} }{c_{ij}}  \leq  \ 1 \nonumber
\end{equation}
we have
\begin{align}
\prod\limits_{X_j\in\Gamma(X_i)}^{}(1-\probacond{.}{.}{c_{ij}}) & \leq 1 - \probacond{.}{.}{c_{ij}} \nonumber
\end{align}
Considering the maximum value of $\probacond{r_{k_j}(D_j)}{\ACincons{w} }{c_{ij}}$:
\begin{equation}
\prod\limits_{X_j\in\Gamma(X_i)}^{}(1-\probacond{.}{.}{c_{ij}}) \leq 1 -  \max_{\ X_j\in\Gamma(X_i)}^{} \{\probacond{.}{.}{c_{ij}}\} \nonumber
\end{equation}

\begin{equation}
\probacond{\forall X_j\in\Gamma(X_i), r_{k_j}(D_j)}{\ACincons{w} }{N} \geq \max_{\ X_j\in\Gamma(X_i)}^{} \{\probacond{.}{.}{c_{ij}}\} \nonumber
\end{equation}
\qed
\end{proof}

We now can bound the expectation of the number of arc-inconsistent values in the domains, after a domain reduction.

\begin{coroll}\label{ExpectedValueBound}
For a domain $D_i$ associated with a variable $X_{i}$, the expected value $E(\tirage{D}_i)$ of arc-inconsistent values in $D_i$ knowing that $k_j$ values have been removed from the domain of the variables $X_j \in \Gamma(X_i)$ is bounded by:
\begin{equation}
 |D_i| \geq E(\tirage{D}_i) \geq \sum\limits_{w_i} \max_{\ X_j\in\Gamma(X_i)}^{} \{\probacond{r_{k_j}(D_j)}{\ACincons{w}_i }{c_{ij}} \}
 \end{equation}
\end{coroll}
\begin{proof}
Straightforward from Prop.~\ref{expectationDomain} and Cor.~\ref{probaValueNet}.
\qed
\end{proof}

\begin{coroll}\label{ExpectedValuesNetworkBound}
The expected value $E(\tirage{N})$ of arc-inconsistent values in the constraint network $N$ is greater or equal than:
\begin{equation}
 E(\tirage{N}) \geq 
 \sum\limits_{X_i\in\mathcal{X}} \sum\limits_{w_i\in D_i} \max_{\ X_j\in\Gamma(X_i)}^{} \{\probacond{r_{k_j}(D_j)}{\ACincons{w}_i \in D_i}{c_{ij}} \}
 \end{equation}
\end{coroll}
\begin{proof}
Conclusion from Proposition~\ref{expectationNetwork} and Corollary~\ref{probaValueNet} is straightforward.
\qed
\end{proof}

To sum up, we now have a formula providing a bound for the number of values that are expected to be removed in the constraint network, under the hypothesis of a domain modification.
\subsection{Propagation of the probabilistic information}
Given a constraint network $N$, Algorithm~\ref{ProbaAC} details how to compute the probability for each value of each domain to be arc-inconsistent (Corollary~\ref{probaValueNet}), and the lower bound of the expected number of arc-inconsistent values for all the domains (Corollary\ref{ExpectedValueBound}).
Algorithm~\ref{ProbaAC} is an adaptation of a coarse grained AC algorithm, AC3~\cite{Mackworth77}. 

\begin{algorithm}[!t]
\small
\caption{$\texttt{ProbAC}(\varAlgo{},\domAlgo{},\cstrAlgo{},\remAlgo{},\{\solAlgo{1}{}{},\solAlgo{2}{}{}\dots \solAlgo{n}{}{}\})$\label{ProbaAC}}
\begin{algorithmic}[1]
\Require 
	\State $\varAlgo{},\domAlgo{},\cstrAlgo{}$: variables, domains and constraints associated with the CSP
	\State $\remAlgo{}$: table of integers - $\remAlgo{i}$ represents the number of values that we assume to be removed (hypothesis) from the domain $\domAlgo{i}$ of $\varAlgo{i}$
	\State $\solAlgo{i}{}{}$: matrix of integers associated to $X_{i}\in \mathcal{X}$ - $\solAlgo{i}{w}{j}$ represents the number of supports ($|\pi_j(S_{ij}^{w})|$) of $w\in D_{i}$ on the constraint $C_{ij}$
\Ensure 
	\State $\probaAlgo{}{}$: $\probaAlgo{i}{j}$ represents a lower bound of the probability for the value $j$ in the domain of $\varAlgo{i}$ to be arc-inconsistent \Comment{Corollary~\ref{probaValueNet}}
	\State $\espAlgo{}$: $\espAlgo{i}$ represents a lower bound of the expected value of arc-inconsistent values in $\domAlgo{i}$ of $\varAlgo{i}$ \Comment{Corollary~\ref{ExpectedValueBound}}  
	\State $\espNetwork{}$ represents a lower bound of the expected value of arc-inconsistent values in the constraint network \Comment{Corollary~\ref{ExpectedValuesNetworkBound}} 
	\Statex
	\State Set of variables $\mathtt{S} \leftarrow \emptyset$ \label{startinit}; $\espNetwork{} \leftarrow 0$
	\ForAll{$\varAlgo{i} \in \varAlgo{}$}
	    \State $\espAlgo{i} \leftarrow 0$ \label{initExpValue}
	    \If{$\remAlgo{i} > 0$}
	   	\State $\mathtt{S}.\mathtt{add}(\varAlgo{i})$ \label{queue}
	    \EndIf
	    \ForAll{$v_i \in \domAlgo{i}$}
		  \State $\probaAlgo{i}{v_i} \leftarrow 0$ \label{initProba}
	    \EndFor
	\EndFor \label{endinit}
	\While{$\mathtt{S} \neq \emptyset$}\label{mainLoopStart}
	      \State Variable $x_i \leftarrow \mathtt{S}.\mathtt{remove}()$
	      \ForAll{$x_j\in \Gamma(x_i)$}
			\ForAll{$v_j\in \domAlgo{j}$} \label{loopValue}
				\State $\mathtt{double}\ \tmp \leftarrow \probacond{\remAlgo{i}}{\ACincons{v}_j}{c_{ij}}$ \Comment{Proposition~\ref{probaValue}} \label{prop5}
				\If{$\tmp > \probaAlgo{j}{v_j}$} \label{updateProbaMax}
					\State $\espNetwork{} \leftarrow \espNetwork{} -  \probaAlgo{j}{v_j} + \tmp$  \label{updateWholeNet} \Comment{Corollary~\ref{ExpectedValuesNetworkBound}}
					\State $\espAlgo{j} \leftarrow \espAlgo{j} -  \probaAlgo{j}{v_j} + \tmp$ \label{updateExpValue}  \Comment{Corollary~\ref{ExpectedValueBound}}
					\State $\probaAlgo{j}{v_j} \leftarrow \tmp$ \label{updateProba} \Comment{Corollary~\ref{probaValueNet}} 
				\EndIf
			\EndFor
			\If{$\lfloor\espAlgo{j}\rfloor > \remAlgo{j} $} \label{enqueueStart}
				\State $\remAlgo{j} \leftarrow \lfloor\espAlgo{j}\rfloor$  \label{updateRemValue}
				\State $\mathtt{S}.\mathtt{add}(\varAlgo{j})$  \label{enqueueEnd}
			\EndIf
		\EndFor
	\EndWhile\label{mainLoopEnd}
\end{algorithmic}
\end{algorithm}

At the initialization step, from Lines~\ref{startinit} to~\ref{endinit}, Algorithm~\ref{ProbaAC} is initialized. Two elements have to be noticed: First, Line~\ref{queue} populates the set of variables to analyze (i.e., each variable for which it is assumed its domain has been modified), according to the table $\remAlgo{}$; 
Second, Lines~\ref{initExpValue} and~\ref{initProba} respectively initialized the expected value of arc-inconsistent values in $D_i$ 
and, the default probability of being arc-inconsistent for each pair variable/value.
Next, the computation of the expected results is ensured as follows.
Line~\ref{prop5} calls the function $\probacondsign{c_{ij}}$ that computes the probability for a given value $v_j \in D_j$ to be arc-inconsistent for a constraint $C_{ij}$ according to the assumption that $\remAlgo{i}$ values are assumed to be removed from $D_i$ (Proposition~\ref{probaValue}).
Lines~\ref{updateProbaMax} and~\ref{updateProba} allow to aggregate the previous information by maintaining the maximum value in the neighborhood of the variable $x_j$ according to formula provided by Corollary~\ref{probaValueNet}. 
Next, Lines~\ref{updateProbaMax} and~\ref{updateExpValue} compute a lower bound on the expected value of arc-inconsistent values in the whole domain $D_j$ according to formula provided by Corollary~\ref{ExpectedValueBound}. 
Next, Lines~\ref{updateProbaMax} and~\ref{updateWholeNet} aggregate the information of Line~\ref{updateProbaMax} to provide a lower bound of the total number of expected arc-inconsistent values for the whole network (Corollary~\ref{ExpectedValuesNetworkBound}).
Finally, Lines~\ref{enqueueStart} to~\ref{enqueueEnd} manages the propagation of the information, by observing the evolution of the expected value of arc-inconsistent values in the domain $D_j$. 

The termination of Algorithm~\ref{ProbaAC} is demonstrated 
from lemma \ref{lemmaTermination5} and line \ref{mainLoopStart}: it is ensured that $\mathtt{S}$ becomes empty and no variable enters $\mathtt{S}$. 

\begin{lemma}\label{lemmaTermination5}
A variable $X_i$ enters the set $\mathtt{S}$ at most $|D_i|$ times. 
\end{lemma}

\begin{proof}
Lines \ref{enqueueStart}-\ref{enqueueEnd} of Algorithm~\ref{ProbaAC} ensure that:
a variable $X_i$ enters the set $\mathtt{S}$ if and only if the number of values $\remAlgo{i}$ assumed to be removed from $D_i$ increases;
for each variable $X_i$, the number of values $\remAlgo{i}$ assumed to be removed is monotonously increasing;
for each variable $X_i$, $|D_i|$ is an upper bound for the number of values $\remAlgo{i}$ assumed to be removed.
The last bullet has to be detailed. The number of values $\remAlgo{i}$ assumed to be removed from $D_i$ is computed as the integer portion $\lfloor \espAlgo{i}\rfloor$ of the expected value $\espAlgo{i}$. That is, the expected value $\espAlgo{i}$ is an upper bound for $\remAlgo{i}$. In addition, $|D_i|$ is an upper bound for $\espAlgo{i}$ (corollary \ref{ExpectedValueBound}). Thus, $|D_i|$ is an upper bound for $\remAlgo{i}$.
\qed
\end{proof}

The initialization step (lines~\ref{startinit} to~\ref{endinit}), has a time complexity $O(nd)$.
Next, at each run of the main loop algorithm (lines~\ref{mainLoopStart} to~\ref{mainLoopEnd}), 
a variable $X_i$ is removed from the set $\mathtt{S}$ then the Proposition~\ref{probaValue} is evaluated
for each $C_{ij}$ and each $v_j \in D_j$. 

\begin{proposition}
Algorithm~\ref{ProbaAC} has a worst-case time complexity $O(nmd^3)$,  
where $m= \max\limits_{i=1..n}\{|\Gamma(X_i)|\}$.
\end{proposition}

\begin{proof}  
Each time a variable $X_i$ enters $S$, the probability $\probacond{\remAlgo{i}}{\ACincons{v}_j}{c_{ij}}$ is evaluated 
for each value $v_j$ in the domain $D_j$ of each variable $X_j\in\Gamma(X_i)$.
In addition, a variable $X_i$ enters $S$ each time the lower bound $\remAlgo{i}$ of the expected value $E(\tirage{X}_i)$ increases. 
In the worst case, the lower bound $\remAlgo{i}$ increases each time by at most one. So, a variable $X_i$ enters $S$ at most $d$ times. 
Therefore, For each of the $n$ variables $X_i$ in the constraint network, 
the probability $\probacond{\remAlgo{i}}{\ACincons{v}_j}{c_{ij}}$ is evaluated at most $m_id^2$ times, where $\nobreak{m_i=|\Gamma(X_i)|}$.
Furthermore, evaluate the probability has a worst case time complexity of $O(d)$. 
Thus, algorithm~\ref{ProbaAC} has a worst case time complexity of $O(nmd^3)$, where $m=\max\limits_{i=1..n}\{m_i\}$. 
\qed
\end{proof}

\section{Conclusion} \label{conclusion}
This work has presented a probabilistic\nobreakdash-based model for classical binary CSPs. This is an original point of view which provides a fine grained analysis of the constraint network that  allows us to better understand both the macro-structure (i.e., interactions between variables through the constraints) and the micro-structure (i.e., interactions between compatible values) of a binary CSP.

\bibliographystyle{splncs03}
\bibliography{ijcai16}

\end{document}